\documentclass[a4paper]{article}

\usepackage[margin=3cm]{geometry}

\usepackage[round]{natbib}
\usepackage[section]{placeins}
\usepackage{enumitem}

\usepackage{amsmath}
\usepackage{amssymb}
\usepackage{amsthm}
\usepackage{stmaryrd}
\usepackage{algpseudocode}
\newtheorem{thm}{Theorem}
\newtheorem{lem}{Lemma}

\newtheorem*{cor*}{\hypertarget{cor}{Corollary}}
\theoremstyle{definition}
\newtheorem{defn}{Definition}
\newtheorem*{rem*}{Remark}

\newcommand{\Sig}{\Sigma^\textrm{P}_2}

\newcommand{\xp}{x^\prime}
\newcommand{\yp}{y^\prime}
\newcommand{\xs}{x^\star}

\newcommand{\X}{\mathcal{X}}
\newcommand{\Y}{\mathcal{Y}}

\newcommand{\bitn}{\{0,1\}^n}
\newcommand{\bitm}{\{0,1\}^m}
\newcommand{\bitnm}{\bitn\times\bitm}

\newcommand{\vp}{v^\prime}

\newcommand{\TAUT}{\textsc{tautology}}

\usepackage{tikz}
\usetikzlibrary{positioning,arrows.meta,calc,backgrounds}
\tikzstyle{strut}=[execute at begin node=\strut]
\tikzstyle{arr}=[-{Straight Barb[length=.7mm]}]
\tikzstyle{rra}=[{Straight Barb[length=.7mm]}-]

\usepackage[colorlinks,allcolors=blue]{hyperref}
\usepackage[nameinlink,capitalize]{cleveref}  
\crefname{equation}{}{}
\crefformat{section}{#2\S#1#3}
\crefformat{subsection}{#2\S#1#3}
\crefmultiformat{section}{#2\S#1#3}{ and~#2\S#1#3}{, #2\S#1#3}{, and~#2\S#1#3}
\crefmultiformat{subsection}{#2\S#1#3}{ and~#2\S#1#3}{, #2\S#1#3}{, and~#2\S#1#3}
\crefrangeformat{section}{#3\S#1#4--#5#2#6}
\crefrangeformat{subsection}{#3\S#1#4--#5#2#6}

\crefformat{thm}{#2Theorem~#1#3}
\crefformat{lem}{#2Lemma~1#3}
\crefformat{defn}{#2Definition~#1#3}
\crefmultiformat{thm}{#2Theorems #1#3}{ and~#2#1#3}{, #2#1#3}{, and~#2#1#3}
\crefmultiformat{lem}{#2Lemmas #1#3}{ and~#2#1#3}{, #2#1#3}{, and~#2#1#3}
\crefmultiformat{defn}{#2Definitions #1#3}{ and~#2#1#3}{, #2#1#3}{, and~#2#1#3}
\newcommand{\citex}[3][]{\citep[#2:][#1]{#3}}  
\newcommand{\citeg}[2][]{\citex[#1]{for example}{#2}}  

\newcommand{\sortbib}[1]{}

\usepackage{titlesec}
\titleformat*{\subsection}{\normalsize\bfseries}

\title{On the Complexity of the Compatibility Problem for Succinctly Encoded Conditional Distributions}
\author{Guy Emerson}
\date{}

\begin{document}

\maketitle

\abstract{%
The motivation for this paper is the investigation of
the trade-offs implicit in probabilistic models used in machine learning.
Models are often used to make predictions
in the form of conditional probabilities.
However, a pair of conditional distributions $p(x|y)$ and $p(y|x)$
may not be compatible with any joint distribution $p(x,y)$.
Given two such conditionals,
determining if there exists a compatible joint
is known as the compatibility problem.
For discrete random variables,
when the conditionals are encoded as probability tables,
the compatibility problem has a known solution,
which is computationally tractable.
In this paper, we formalise and study a succinct version of the problem,
encoding conditional distributions as arithmetic circuits.
This is applicable to practical applications of probabilistic modelling
in high-dimensional settings, including neural network models.
We show that,
for succinct circuit representations of conditionals,
the compatibility problem is intractable.
In the case that all probabilities are non-zero,
the problem is co-NP-complete.
In the case that probabilities can be zero,
we give examples to demonstrate that
several notions of compatibility can be distinguished,
and we prove that multiple versions of the problem are PSPACE-complete.
Furthermore, we show that,
assuming the polynomial hierarchy does not collapse,
there exist compatible succinct conditionals
whose joint cannot be expressed succinctly.
Implications of these results
for probabilistic modelling and machine learning are discussed.}

\section{Introduction}

A joint distribution $p(x,y)$ over $\X\times\Y$
implicitly defines conditional distributions
$p(x|y)=\frac{p(x,y)}{\sum_{\xp}p(\xp,y)}$
and $p(y|x)=\frac{p(x,y)}{\sum_{\yp}p(x,\yp)}$.
Conversely, determining whether conditional distributions
$p(x|y)$ and $p(y|x)$
are compatible with some joint distribution $p(x,y)$
is known as the compatibility problem
\citep{arnold1989compatible}.
It can be solved as a constraint satisfaction problem
using techniques from linear algebra
\citep{arnold2004partial}.

However, if either of the spaces $\X$ and $\Y$ are large,
iterating over all points is computationally expensive.
A number of authors have sought more efficient algorithms
\citeg{tian2008unified,ip2009canonical,chen2010compatible,kuo2011compatible,yao2014compatible,kuo2017highdim,miranda2020rip}.
However, the question remains how far the cost can be reduced.

In fact, in high-dimensional settings, it would be prohibitive
to even store a distribution as an enumeration of all values.
In practice, a probabilistic model might
store only a small number of parameters,
which can be used to calculate any particular probability.
In other words, the probabilistic model is a \textit{succinct} representation,
in the sense of \citet{galperin1983succinct},
but applied to distributions rather than graphs.
However, to our best knowledge,
the compatibility problem has not been studied in the succinct setting.
In this paper, we formalise such a probability distribution
as a succinct arithmetic circuit \citep{darwiche2003poly}
with a further constraint that all values can be expressed succinctly
(that is, with a bounded number of digits).

This setting includes Bayesian networks,
where inference is known to be \#P-complete
\citep{roth1996hard}.
At first sight, it seems possible
that the compatibility problem might be simpler than Bayesian inference,
since the conditionals are provided as input,
and do not need to be computed from scratch.

The main result of this paper is that
the succinct compatibility problem is intractable.
If all probabilities are non-zero,
the problem is co-NP-complete,
because incompatibilities can be identified locally.
If probabilities can be zero,
the non-zero conditional probabilities define
a directed graph over $\X\times\Y$.
This graph may have multiple strongly connected components,
and we could require compatibility in all components (strong compatibility),
or in at least one (weak compatibility).
We show the strong and weak compatibility problems
are both PSPACE-complete.

If a joint exists,
we may also wish to know whether it can be expressed succinctly.
We define two notions of tractable compatibility,
a weaker one requiring that each probability value
requires only polynomially many digits to express,
and a stronger one requiring that the joint
can be represented as a succinct circuit.
We show that determining compatibility with a tractably valued joint
is still PSPACE-complete,
while determining compatibility with a succinctly expressible joint
is only $\Sig$-complete.
Assuming that $\Sig\neq\textrm{PSPACE}$,
these results mean that there are families of compatible succinct conditionals
whose joints have succinctly expressible values
but which nonetheless cannot be represented as succinct circuits.

\section{Succinct Circuit Representations of Distributions}

Consider random variables over boolean vectors $\bitn$.
We define a circuit representation of a distribution
as an arithmetic circuit \citep{darwiche2003poly},
which takes the values of the random variable(s) as input
and gives the probability as output.
More precisely, an arithmetic circuit
is a directed acyclic graph,
where each node is labelled as
an input (with no incoming edges),
an integer constant (with no incoming edges),
or an operation ($+$ or $\times$, with exactly two incoming edges);
some nodes can be additionally labelled as output.

Rational numbers can be represented as pairs of integers,
and so an arithmetic circuit with two output nodes
can be used to represent a probability distribution.
To represent a conditional distribution $p(x|y)$,
both $x$ and $y$ are considered as input,
using $n$ input nodes to represent a random variable over $\bitn$.
However, for such a circuit to represent a tractable computation,
we must also ensure that each node can be tractably computed.
As a counter-example, consider a circuit
structured as a long chain,
starting from the constant $2$,
with each subsequent step in the chain
multiplying the previous number with itself;
after $n$ multiplications, this circuit computes
the double exponential $2^{2^n}$,
which requires exponentially many bits of memory.

\begin{defn}
A \textbf{tractable circuit representation} of a distribution
is an arithmetic circuit with exactly two output nodes,
such that, given any values in $\{0,1\}$ for the input nodes,
the number of bits in the binary representation of each node
is bounded by the size of the circuit.%
\label{defn:tract-circ}
\end{defn}

Such arithmetic circuits are flexible enough
to simulate arbitrary boolean circuits.
By discretising real numbers to any given degree of accuracy,
we can also simulate numerical algorithms.
In particular, by introducing \textit{max} as an additional operation
that finds the maximum of two nodes,
tractable arithmetic circuits can simulate
neural networks with piecewise polynomial activation functions
(for example, ReLU).

\begin{defn}
A family of distributions is \textbf{succinctly representable}
if the distributions can be represented by tractable circuits
of polynomially bounded size with respect to the number of input nodes.%
\label{defn:succ-fam}
\end{defn}

\cref{defn:tract-circ} uses a linear bound on bit lengths,
but any polynomial bound could be used instead,
in order to define the same families of distributions
in \cref{defn:succ-fam}:
given a family of circuits whose binary node values
have lengths bounded by a polynomial in circuit size,
we can construct a ``padded'' family of circuits satisfying
\cref{defn:tract-circ},
which are only polynomially larger,
and whose outputs always match the original family.

\section{The Succinct Compatibility Problem}
\label{sec:nonzero}

We first restrict ourselves
to distributions that are non-zero everywhere.
This avoids the problem that
conditional probabilities are undefined
when conditioning on an event with probability zero:
$p(y|x)$ is undefined if $p(x)=0$,
and $p(x|y)$ is undefined if $p(y)=0$.
The case where probabilities can be zero
will be considered in~\cref{sec:zero}.
In the non-zero case, we have the following result.\footnote{%
	This is naturally defined as a promise problem,
	since a circuit representation has no guarantee
	that the outputs are normalised to sum to~$1$.
	However, the compatibility problem can be extended to all inputs
	if we follow the convention that
	a circuit implicitly defines a distribution
	after its outputs are normalised.
	Ratios of probabilities are invariant to normalisation,
	and so this does not change any of the complexity results in this paper.
}

\begin{thm}
Given tractable circuit representations of the conditionals $p(x|y)$ and $p(y|x)$,
with $p(x|y),p(y|x)>0$ for all $(x,y)$,
determining whether they are
compatible with some joint $p(x,y)$ is co-NP-complete.
\label{thm:nonzero}
\end{thm}

The proof is given in~\cref{sec:nonzero-lower,sec:nonzero-upper}
(for the upper and lower complexity bounds, respectively).
The basic idea is that, in order to check compatibility,
we need to check a constraint
on each set of four points $(x_1,x_2,y_1,y_2)$,
following \citet{arnold1994uniform}.
Intuitively, this is a co-NP-complete problem
because we have exponentially many constraints to check,
and any single violation means the conditionals are incompatible.

\subsection{Proof of upper bound: Non-zero Succinct Compatibility is in co-NP}
\label{sec:nonzero-upper}

We define a verifier, with certificates of incompatibility
of the form $(x_1,x_2,y_1,y_2)$.
A certificate can be verified by checking:
\begin{equation}
p(x_1|y_1)p(y_1|x_2)p(x_2|y_2)p(y_2|x_1)
\neq
p(y_1|x_1)p(x_1|y_2)p(y_2|x_2)p(x_2|y_1)
\label{eqn:verify}
\end{equation}

This inequality can be checked in linear time
with respect to the maximum circuit size
of $p(x|y)$ and $p(y|x)$.
We must now show that this verifier is correct.

First, suppose compatibility.
Then a joint distribution and marginal distributions exist,
such that $p(x|y)p(y)=p(x,y)=p(y|x)p(x)$ for any $(x,y)$.
Therefore, for any $(x_1,x_2,y_1,y_2)$:
\begin{equation}
\begin{split}
p(x_1|y_1)p(y_1|x_2)p(x_2|y_2)p(y_2|x_1)p(x_1)p(x_2)p(y_1)p(y_2)\\
= p(y_1|x_1)p(x_1|y_2)p(y_2|x_2)p(x_2|y_1)p(x_1)p(x_2)p(y_1)p(y_2)
\end{split}
\label{eqn:joint}
\end{equation}

If any marginal probability $p(a)$ is zero,
conditional probabilities $p(a|b)$ would also be zero.
The marginal terms are therefore non-zero and cancel,
so the verifier rejects any certificate.

Conversely, suppose every certificate is rejected.
Then fix some $(x_0,y_0)$.
A joint distribution can be constructed as the following,
which is well-defined since all conditional probabilities are non-zero:
\begin{equation}
p(x,y)=\alpha\frac{p(x|y)p(y|x_0)}{p(x_0|y)p(y_0|x_0)}
\end{equation}

The value of the normalisation constant
$\alpha=p(x_0,y_0)$ is defined
by summing over all~$(x,y)$
to ensure $\sum p(x,y)=1$.
We must now show this joint distribution
is compatible with both conditionals.
In the above expression, the only term that depends on~$x$ is $p(x|y)$,
and so
$\frac{p(x,y)}{\sum_{\xp} p(\xp,y)}
=\frac{p(x|y)}{\sum_{\xp} p(\xp|y)}
=p(x|y)$ as required.

Furthermore, for every $(x,y)$,
the certificate $(x,x_0,y,y_0)$ is rejected,
and so:
\begin{equation}
p(x,y)=\alpha\frac{p(y|x)p(x|y_0)}{p(y_0|x)p(x_0|y_0)}
\end{equation}

Here, the only term that depends on $y$ is $p(y|x)$,
and so we also have
$\frac{p(x,y)}{\sum_{\yp} p(x,\yp)}=p(y|x)$.

Therefore, all certificates are rejected
iff $p(x|y)$ and $p(y|x)$ are compatible.

\subsection{Proof of lower bound: Non-zero Succinct Compatibility is co-NP-hard}
\label{sec:nonzero-lower}

We show this by reduction from \TAUT,
the co-NP-complete problem of determining whether a given boolean formula
is true under every assignment of values to its variables.
To define a reduction, the idea is to use $x$ to hold the variables,
and to adjust the distribution of $p(y|x)$
depending on the value of the formula,
so that any false value creates an incompatibility.

Given a boolean formula~$\varphi$ with $n$ variables,
let $\X=\Y=\bitn$,
let $p(x|y)=\frac{1}{2^n}$, and let:
\begin{equation}
p(y|x)=
\begin{cases}
\frac{1}{2^n} & \textrm{if $\varphi(x)=\top$}\\
\frac{2}{2^n+1} & \textrm{if $\varphi(x)=\bot$ and $x=y$}\\
\frac{1}{2^n+1} & \textrm{if $\varphi(x)=\bot$ and $x\neq y$}
\end{cases}
\end{equation}

By construction, both $p(y|x)$ and $p(x|y)$
are representable by circuits linear in the length of~$\varphi$,
and the reduction can be performed in linear time.

If $\varphi$ is a tautology,
then $p(y|x)$ and $p(x|y)$ are compatible,
with uniform joint distribution $p(x,y)=\frac{1}{2^{2n}}$.

Conversely, suppose there is some $\xs$
for which $\varphi(\xs)=\bot$.
Choose any $x_0\neq \xs$.
Then the certificate $(x_0,\xs,x_0,\xs)$
is rejected by~\cref{eqn:verify},
since the $p(x|y)$ terms cancel,
and for the $p(y|x)$ terms, we have
$p(x_0|x_0)\geq p(\xs|x_0)$, and $p(\xs|\xs)>p(x_0|\xs)$.

Therefore,
$\varphi$ is a tautology
iff $p(x|y)$ and $p(y|x)$ are compatible.

\section{Succinct Compatibility with Zeroes}
\label{sec:zero}

If at least one term on each side of~\cref{eqn:verify} is zero,
then we immediately have equality.
\citet{arnold1998nearly} give the following counter-example
of incompatible conditionals
where we always have equality in~\cref{eqn:verify},
showing that further constraints are necessary.
\begin{align}
p(y|x) &=\begin{pmatrix}
	\frac{1}{2} & \frac{1}{2} & 0 \\[.5ex]
	0 & \frac{1}{2} & \frac{1}{2} \\[.5ex]
	\frac{1}{2} & 0 & \frac{1}{2}
\end{pmatrix}
&p(x|y) &=\begin{pmatrix}
	\frac{1}{3} & \frac{2}{3} & 0 \\[.5ex]
	0 & \frac{1}{3} & \frac{2}{3} \\[.5ex]
	\frac{2}{3} & 0 & \frac{1}{3}
\end{pmatrix}
\label{eqn:eg-incomp}
\end{align}

In fact, when we allow zero probabilities,
we have more than one way to define compatibility,
because conditioning on an event with zero probability is undefined.
Following \citet{walley1991statistical}
and \citet{miranda2009coherence},
we can distinguish strong and weak compatibility,
which differ in whether or not the conditionals are constrained
in the cases where the joint assigns zero probability
to the conditioning event:\footnote{%
	These definitions allow the joint
	not to have full support in its margins,
	and allow the conditionals to formally sum to 0
	when the conditioning event has zero probability.
	Restricting the input conditionals to disallow this
	would not affect the complexity results.
}

\begin{defn}
Conditionals $p(x|y)$ and $p(y|x)$ are \textbf{weakly compatible}
if there exists a joint $p(x,y)$
such that
$p(x|y)=\frac{p(x,y)}{\sum_{\xp}p(\xp,y)}$
and $p(y|x)=\frac{p(x,y)}{\sum_{\yp}p(x,\yp)}$
whenever these values are defined.
\end{defn}

\begin{defn}
Conditionals $p(x|y)$ and $p(y|x)$ are \textbf{strongly compatible}
if there exists a joint $p(x,y)$
such that
$p(x|y)=\frac{p(x,y)}{\sum_{\xp}p(\xp,y)}$
and $p(y|x)=\frac{p(x,y)}{\sum_{\yp}p(x,\yp)}$
whenever these values are defined,
and $p(x|y)=p(y|x)=0$ otherwise.
\end{defn}

Because of the special behaviour of zero probabilities,
it is convenient to view the conditionals as defining a graph,
where edges correspond to non-zero probabilities:

\begin{defn}
Given $p(x|y)$ and $p(y|x)$,
their \textbf{non-zero graph}
is defined by
the set of nodes $\{(x,y)\ |\ p(y|x)>0\textrm{ or }p(x|y)>0\}$,
with directed edges
$(x,y)\rightarrow(x,\yp)$ whenever $p(\yp|x)>0$
and $(x,y)\rightarrow(\xp,y)$ whenever $p(\xp|y)>0$.
\end{defn}

If forward and reverse edges exists,
we have a finite probability ratio:

\begin{defn}
The \textbf{probability ratio} of an edge $(x,y)\rightarrow(x,\yp)$
is $\frac{p(\yp|x)}{p(y|x)}$,
and of an edge $(x,y)\rightarrow(\xp,y)$
is $\frac{p(\xp|y)}{p(x|y)}$.
The \textbf{probability ratio} of a path
is the product of its edge ratios.
\end{defn}

For strongly compatible conditionals,
if two paths start and end at the same nodes,
they must have the same ratio,
equal to the ratio of joint probabilities
(as observed by
\citealp{besag1994markov,kuo2011compatible}).
Therefore, all loops (closed paths)
must have ratio~$1$.
In fact, this is also sufficient,
as stated below in~\cref{lem:loop-ratio}.
For example, in~\cref{eqn:eg-incomp},
there is exactly one non-trivial loop (without repeating nodes),
and so it is necessary and sufficient to check the following constraint,
a product of six probability ratios:
\begin{equation}
\frac
{p(y_2|x_1)p(x_2|y_2)p(y_3|x_2)p(x_3|y_3)p(y_1|x_3)p(x_1|y_1)}
{p(x_1|y_2)p(y_2|x_2)p(x_2|y_3)p(y_3|x_3)p(x_3|y_1)p(y_1|x_1)}
= 1
\label{eqn:verify-loop}
\end{equation}

\begin{lem}
Conditionals $p(x|y)$ and $p(y|x)$ are strongly compatible
iff every edge in the non-zero graph has a reverse edge
and every loop has probability ratio~$1$;
they are weakly compatible
iff there is some strongly connected component of the non-zero graph
where every outgoing edge has a reverse edge
and every loop has probability ratio~1.%
\label{lem:loop-ratio}
\end{lem}

\begin{proof}
The strongly compatible case is proved by \citet{slavkovic2006toric}
(our probability ratios correspond to their binomials).
It can also be proved by generalising
the construction in~\cref{sec:nonzero-upper},
constructing the joint in terms of probability ratios of paths,
and fixing not just one point
but one point in each connected component of the graph.

Now suppose the conditionals are weakly compatible.
The joint must be non-zero at some node.
Furthermore, that node's strongly connected component
cannot have edges to other components
(if an edge $a\rightarrow b$
does not have a reverse edge $b\rightarrow a$,
the joint must have $p(a)=0$,
contradicting that the joint is non-zero in the component).
Restricting the conditionals to this strongly connected component
is therefore well-defined and gives the strongly compatible case,
and so all loops in the component have ratio~1.
Conversely, suppose there is some strongly connected
where all loops have ratio~1.
Restricting the conditionals to this component
gives the strongly compatible case,
and so we have a compatible joint over this component.
This extends to a weakly compatible joint over the whole space,
by setting the joint to~0 outside the component.
\end{proof}

An example is given in~\cref{eqn:eg-weak},
using rows for $x$ and columns for $y$,
showing weakly compatible conditionals
with two different violations of strong compatibility.
There are three strongly connected components,
$\{(1,1),(1,2),(2,1),(2,2)\}$,
$\{(3,3),(3,4),(4,3)\}$,
and $\{(4,4),(4,5),(5,4),(5,5)\}$.
The first component has inequality in~\cref{eqn:verify},
while the second component has directed edges
to the third (but not vice versa).
Either of these violations is sufficient to show that
the conditionals are not strongly compatible;
for weak compatibility, the joint must assign
zero probability to both the first and second components.
The third component has no violations of either type
and this is sufficient to show that
the conditionals are weakly compatible.
\begin{align}
p(y|x) &=\begin{pmatrix}
	\frac{1}{2} & \frac{1}{2} & 0 & 0 & 0 \\[.5ex]
	\frac{1}{2} & \frac{1}{2} & 0 & 0 & 0 \\[.5ex]
	0 & 0 & \frac{1}{2} & \frac{1}{2} & 0 \\[.5ex]
	0 & 0 & 0 & \frac{1}{2} & \frac{1}{2} \\[.5ex]
	0 & 0 & 0 & \frac{1}{2} & \frac{1}{2}
\end{pmatrix}
&p(x|y) &=\begin{pmatrix}
	\frac{2}{3} & \frac{1}{3} & 0 & 0 & 0 \\[.5ex]
	\frac{1}{3} & \frac{2}{3} & 0 & 0 & 0 \\[.5ex]
	0 & 0 & \frac{1}{2} & 0 & 0 \\[.5ex]
	0 & 0 & \frac{1}{2} & \frac{1}{2} & \frac{1}{2} \\[.5ex]
	0 & 0 & 0 & \frac{1}{2} & \frac{1}{2}
\end{pmatrix}
\label{eqn:eg-weak}
\end{align}

We have the following result, with the proof given in~\cref{sec:zero-upper,sec:zero-lower}:
\begin{thm}
Given tractable circuit representations of the conditionals $p(x|y)$ and $p(y|x)$,
determining whether they are
strongly (weakly) compatible is PSPACE-complete.
\label{thm:zero}
\end{thm}

However, the existence of a joint does not mean that
the joint is also succinctly representable.
For example, consider the following conditionals,
where $x,y\in\bitn$,
and for ease of notation,
we view $x$ and $y$ as integers represented in binary:
\begin{align}
p(y|x)&=
\begin{cases}
1 & \textrm{if }x=y=2^n-1\\
\frac{2}{3} & \textrm{else if }y=x\\
\frac{1}{3} & \textrm{else if }y=x+1\\
0 & \textrm{otherwise}
\end{cases}
& p(x|y)&=
\begin{cases}
1 & \textrm{if }x=y=0\\
\frac{2}{3} & \textrm{else if }x=y-1\\
\frac{1}{3} & \textrm{else if }x=y\\
0 & \textrm{otherwise}
\end{cases}
\label{eqn:eg-doub-exp}
\end{align}

These conditionals are compatible,
with non-zero joint probability where $y=x$ or $y=x+1$.
The non-zero graph forms a chain
(so no non-trivial loops exist,
and so there are no corresponding constraints),
with $(x,x)$ connected to $(x,x+1)$ and $(x-1,x)$.
The chain is exponentially long (with $2^{n+1}-1$ nodes),
and each subsequent node has half the previous node's probability,
so the the probability of the final node is doubly exponentially small,
which means the denominator's binary representation
has an exponential number of digits.\footnote{%
	More precisely, we have $p(2^n-1,2^n-1)=\frac{1}{2^{2^{n+1}-1}-1}$.
	More generally, for a polynomially bounded circuit,
	we could have each edge with an exponential ratio $2^{p(n)}$,
	for some polynomial $p(n)$,
	but the probability ratio between the first and last nodes
	is still only doubly exponential, $2^{p(n)(2^{n+1}-2)}$,
	not triply exponential.
}

This motivates two increasingly strict notions of compatibility,
as well as a weaker notion of a succinct family of distributions.\footnote{%
	As with the definitions of tractable circuits
	and succinctly representable families of distributions
	in \cref{defn:tract-circ,defn:succ-fam},
	the linear bounds in \cref{defn:comp-tract-val,defn:comp-tract-joint}
	could be replaced with a polynomial bound
	without affecting \cref{defn:succ-val-fam}.
	Similarly, in \cref{defn:comp-tract-val},
	the more obvious definition would directly use
	probabilities, rather than ratios.
	Using ratios simplifies the statement and proof
	of \cref{thm:tract-valued},
	while giving the same families of distributions.
}

\begin{defn}
Two circuit representations of conditionals
are \textbf{strongly (weakly) compatible with tractable values},
if they are strongly (weakly) compatible with a joint distribution
where every ratio of two joint probabilities
has numerator and denominator with
binary length bounded by the total size of the two circuits.%
\label{defn:comp-tract-val}
\end{defn}

\begin{defn}
Two circuit representations of conditionals
are \textbf{strongly (weakly) compatible with a tractable joint},
if they are strongly (weakly) compatible and the joint distribution
can be represented by a tractable circuit with size
bounded by the total size of the two circuits.%
\label{defn:comp-tract-joint}
\end{defn}

\begin{defn}
A family of distributions is \textbf{succinctly valued}
if the binary length of the numerator and denominator of each probability
is polynomially bounded with respect to the number of dimensions.
\label{defn:succ-val-fam} 
\end{defn}

The conditionals in~\cref{eqn:eg-doub-exp}
are compatible, but not compatible with tractable values.
Compatibility with a tractable joint
implies compatibility with tractable values,
but the converse might not be true.

We have the following two results.
The proof of \cref{thm:tract-valued}
is given in~\cref{sec:zero-upper,sec:zero-lower},
alongside the proof of \cref{thm:zero}.
The proof of \cref{thm:zero-tract}
is given in \cref{sec:tract}.

\begin{thm}
Given tractable circuit representations of the conditionals $p(x|y)$ and $p(y|x)$,
determining whether they are
strongly (weakly) compatible with tractable values
is PSPACE-complete.
\label{thm:tract-valued}
\end{thm}

\begin{thm}
Given tractable circuit representations of the conditionals $p(x|y)$ and $p(y|x)$,
determining whether they are
strongly (weakly) compatible with a tractable joint
is $\Sig$-complete.%
\label{thm:zero-tract}
\end{thm}

\begin{cor*}
Assuming $\Sig\neq\textrm{PSPACE}$,
there are families of succinctly representable conditionals
whose joints are succinctly valued
but not succinctly representable.
\end{cor*}

To prove the complexity upper bounds,
the basic idea is to consider constraints defined by loops,
somewhat similarly to \citet{slavkovic2006toric}.
However, a PSPACE algorithm over succinct circuits
cannot explicitly store arbitrary paths or subgraphs,
since there could be doubly exponentially many of them.
However, there is a high level of redundancy in the set of constraints,
which means that it is not necessary to iterate through all of them.
While the full set of constraints is necessary
to characterise the whole space of compatible conditionals
(as done by \citeauthor{slavkovic2006toric}),
we can determine compatibility of any particular pair of conditionals
by considering a smaller set of constraints.

The proof is simplified by providing a non-deterministic algorithm
which walks over the non-zero graph,
rejecting if we return to the starting node
and the probability ratio is not~$1$.
The redundancy in the constraints allows us
to only keep track of the ends of the path and its probability ratio,
while other information about the path can be forgotten.
A more precise characterisation of the redundancy in the constraints
is given in~\cref{sec:alg},
along with a sketch of a deterministic algorithm.

\subsection{Proof of upper bound: Succinct Compatibility is in PSPACE}
\label{sec:zero-upper}

We have four results to prove,
for strong/weak compatibility with/without tractable values
(\cref{thm:zero,thm:tract-valued}).

We start with strong compatibility with tractable values.
Since $\textrm{NPSPACE}=\textrm{PSPACE}$ \citep{savitch1970nondet},
we define the following non-deterministic algorithm.
Furthermore, since PSPACE is closed under complement,
for ease of exposition we define the algorithm
such that it rejects if a single computation path rejects.

Non-deterministically pick a point $(x,y)\in\bitnm$.
If exactly one of $p(x|y)$ and $p(y|x)$ is zero, reject;
if both are zero, accept.
Otherwise, record $(x,y)$ as the starting node,
initialise $(x,y)$ as the current node of a non-deterministic walk,
initialise a counter (from $0$ to $2^{\min\{n,m\}+1}$),
and initialise $1$ as the current probability ratio,
with the amount of memory for the numerator and denominator
set as the total size of the two input circuits.
At each step of the walk,
non-deterministically choose the next node,
changing only one of $x$ and $y$ in the current node.
If exactly one of the edges to and from that node exists, reject;
if neither exist, accept;
otherwise (when both exist), multiply the current probability ratio
by the edge's probability ratio
(simplifying numerator and denominator if possible),
update the current node of the walk,
and increment the counter.
If the probability ratio exceeds the allowable memory, reject.
If the current node is equal to the starting node,
check whether the probability ratio is equal to~$1$:
accept if it is, and reject if it is not.
If the counter has reached its max value, accept;
otherwise, continue the walk.

If the conditionals are strongly compatible with tractable values,
they are non-zero for the same values
and every loop has probability ratio~$1$
(by \cref{lem:loop-ratio}),
and the probability ratio of every path fits into memory
(be definition).
Therefore, for any computation path, the algorithm accepts.
Conversely,
if the conditionals are not strongly compatible with tractable values,
either there is some $(x,y)$ where exactly one conditional is non-zero,
or there is some loop whose probability ratio is not~$1$
with the length of the loop bounded by $2^{\min\{n,m\}+1}$
(any longer loop can be decomposed into shorter loops),
or some probability ratio cannot fit into memory.
In each case, there is a rejecting computation path.

We now turn to strong compatibility without tractable values.
A path's probability ratio is 1,
iff all terms cancel,
iff all prime factors of those terms cancel.
We know that all values of circuit nodes have at most $k$ digits,
where $k$ is the total circuit size.
and so we only need to consider primes up to $2^k$.
We can amend the algorithm,
non-deterministically choosing a prime~$p$ up to this size,
at the beginning of the algorithm.
Instead of storing the probability ratio,
we store only the exponent of $p$,
i.e.\ the number of times it divides the numerator or denominator,
while all other prime factors are ignored.
Instead of checking whether the probability ratio is~1,
we check whether the exponent of $p$ is~0.
The stored exponent is bounded by
the maximum exponent in one edge ratio ($k$)
times the maximum path length ($2^{\min\{n,m\}+1}<2^k$),
and so it can be stored in polynomial space with respect to~$k$.

Finally, we turn to weak compatibility (with or without tractable values).
The conditionals are weakly compatible
iff there is a strongly connected component
without outgoing edges to other components
where all loops have ratio~1 (\cref{lem:loop-ratio}).
We amend the algorithm to take a starting node as input
(instead of choosing it non-deterministically).
When checking the existence of edges,
an incoming edge without an outgoing edge is not rejected,
but an outgoing edge without an incoming edge is still rejected.
This gives an $\textrm{NPSPACE}=\textrm{PSPACE}$ algorithm for compatibility
of the strongly connected component that includes the starting node.
Since $\textrm{PSPACE}^\textrm{PSPACE}=\textrm{PSPACE}$,
we can iterate over all starting nodes,
and accept if a single one accepts.

\subsection{Proof of lower bound: Succinct Compatibility is PSPACE-hard}
\label{sec:zero-lower}

For circuit representations of graphs,
the connectivity problem and st-connectivity problem
are both is PSPACE-complete
\citep{papadimitriou1986succinct,balcazar1992succinct}.

Given succinctly encoded undirected graph $(V,E)$, define a new graph
with nodes $V \cup (V\times V)$,
and edges between $v$ and $(v,\vp)$ for each $v\vp\in V$
as well as between $(v,\vp)$ and $(\vp,v)$ if $(v,\vp)\in E$.
This graph can also be encoded succinctly,
and it is connected iff the original graph is connected.

Furthermore, the degree of each node in the new graph
can be determined in polynomial time:
each node $v$ has degree $|V|$,
and each node $(v,\vp)$ has degree~2
if $(v,\vp)\in E$ and degree~1 otherwise.
We can therefore define a uniform random walk on the graph,
where the conditional $p(y|x)$ of taking one step
can be calculated in polynomial time.
If we also define $p(x|y)$ in the same way,
the two are compatible, with the joint being
a stationary distribution of the random walk over the graph
(unique iff the graph is connected).

We can reduce connectivity to weak compatibility
by modifying the $p(x|y)$ distribution at one point
(e.g.\ as if some node $v$ has an extra self-connection).
Then the conditionals are not compatible within $v$'s connected component.
Therefore, the conditionals are weakly compatible
iff there is another connected component, i.e.\ the graph is disconnected.

We can reduce st-connectivity to strong compatibility
as follows.
We add an additional node $u$ to the new graph,
adding non-zero probability to $(t,u)$, $(u,u)$, and $(u,s)$.
If $s$ is connected to $t$, this creates a loop
whose probability ratio must be~$1$;
if they are not connected, no new constraints are introduced.
with $p(Y=u|X=t)\propto 1$ (scale down rest),
$p(Y=s|X=u)=p(Y=u|X=u)=\frac{1}{2}$,
$p(X=t|Y=u)=p(X=u|Y=u)=\frac{1}{2}$,
$p(X=u|Y=s)\propto 1$ (scale down rest).
The conditionals are strongly compatible
iff the random walk admits a stationary distribution
where $s$ and $t$ have equal probability.
Now duplicate the entire graph,
but with $p(X=t|Y=u)=\frac{1}{3}$ and $p(X=u|Y=u)=\frac{2}{3}$
in the second copy.
The conditionals are strongly compatible
iff the random walk admits a stationary distribution
where $t$ has twice the probability of $s$.
The random walk admits both stationary distributions
iff $s$ and $t$ are not connected.
Therefore, the conditionals are strongly compatible
iff $s$ and $t$ are not connected.

When compatible, these reductions always give circuits compatible with tractable values,
because, within each connected component, a node's probability
is proportional to the node's degree.
Each probability ratio is therefore equal to
the ratio of the nodes' degrees.
The maximum degree of a node is bounded by the number of nodes,
which is always less than $2^n$,
where $n$ is the number of inputs in the circuit,
so probability ratios have numerator and denominator
both with at most $n$ bits.

\subsection{Proof: Succinct Compatibility with a Tractable Joint is $\Sig$-complete}
\label{sec:tract}

Weak compatibility with a tractable joint is in $\Sig$,
since a necessary and sufficient condition can be expressed
with $\exists\forall$ quantifiers:
there exists a tractable circuit for the joint
and there exists a point $(x_0,y_0)$ with non-zero joint probability,
such that for every triple $(x,y_1,y_2)$ or $(x_1,x_2,y)$,
the conditional ratio agrees with the joint ratio,
i.e.\ $p(y_1|x)p(x,y_2)=p(y_2|x)p(x,y_1)$
and $p(x_1|y)p(x_2,y)=p(x_2|y)p(x_1,y)$, respectively.

Strong compatibility with a tractable joint is similarly in $\Sig$,
but if both joint probabilities are zero,
we require both conditional probabilities to be zero.

Weak compatibility with a tractable joint is $\Sig$-hard,
by reduction from the $\Sig$-hard problem of
determining whether a boolean proposition
$\exists u \forall v \varphi(u,v)$ is true.
Given $\varphi(u,v)$,
we define $x$ and $y$ each of the form $(u,v)$.
We now define conditional distributions such that
each $u$ forms one connected component
(non-zero probability only when $x=(u,v_1)$ and $y=(u,v_2)$),
and a single false value for some $v$
forces incompatibility in that component
(using the same construction as in~\cref{sec:nonzero-lower}).
So $\exists u \forall v \varphi(u,v)$
iff there is some compatible component,
i.e.\ the conditionals are weakly compatible.
If the joint distribution exists, it is uniform, therefore tractable.

Strong compatibility with a tractable joint is also $\Sig$-hard.
For $\varphi(u,v)$,
we construct a potentially long loop,
where each $u$ is one potential link,
and any false $(u,v)$ constructs the link
(with probability ratio~$1$);
furthermore, we add a potentially contradictory link from start to end
(with probability ratio~$2$).
More precisely, viewing $u$ as encoding an integer from $0$ to $2^n-1$,
we define $x$ and $y$
to each be of the form $(u,v)$ or $u$ or $2^n$.
We define both conditionals to assign non-zero probability mass
to the following points:
$(u,u)$;
$(2^n,2^n)$;
$(2^n,0)$;
$(u,(u,v))$;
$((u,v),u+1)$;
and for each $\varphi(u,v)=\bot$,
$((u,v),(u,v))$.
All distributions are uniform except for $p(x|(0,0))$
where we set $p((2^n,0)|(0,0))=2p((0,0)|(0,0))$.
By construction, $\forall u \exists v \neg\varphi(u,v)$
iff the conditionals are incompatible,
i.e.\ $\exists u \forall v \varphi(u,v)$
iff the conditionals are compatible.
Each conditional is tractable,
as calculating any probability requires at most one evaluation of $\varphi$,
and then the return value is one of
$0$, $1$, $\frac{1}{2}$, or $\frac{1}{2^m+1}$.
If the joint exists, it is tractable:
given a witness $u^\star$,
a (not necessarily unique) compatible joint assigns unnormalised probability
$1$ whenever $u\leq u^\star$ in the $x$ component,
and $2$ otherwise.
(Finding the normalisation constant is not necessary
for determining compatibility;
a tractable circuit with the correct normalisation exists.)

\begin{rem*}
This last construction builds gadgets for existential and universal quantifiers
(logical OR and AND gates ranging over $u$ and $v$ values),
and this can be generalised to give an alternative proof of PSPACE-hardness
(\cref{sec:zero-lower}).
However, the generalised construction does not provide
a tractable joint.
\end{rem*}

\section{Algebraic Structure of Compatibility Constraints}
\label{sec:alg}

Although the set of loops in the non-zero graph
can be doubly exponentially large,
for any problem instance we only need to check at most exponentially many.
We already know from linear algebra that we have
at most $(2^n-1)(2^m-1)$ independent constraints.
The aim of this section is to more precisely characterise
the set of compatibility constraints for any given pair of conditionals.

It is helpful to distinguish constraints
on the minimal loops of four points, as in~\cref{eqn:verify},
from constraints on longer loops, as in~\cref{eqn:verify-loop},
which we formalise in~\cref{defn:local,defn:non-local}.
Intuitively, the computationally hardest part of
the succinct compatibility problem is finding these longer loops,
and it is PSPACE-complete because
path-finding on succinct graphs is PSPACE-complete.
In contrast, the non-zero case is only co-NP-complete
because it is fully characterised by local constraints.

\begin{defn}
Given $p(x|y)$ and $p(y|x)$,
a \textbf{local compatibility constraint}
is the probability ratio of a loop in the non-zero graph of the form
$(x_1,y_1)\rightarrow(x_1,y_2)\rightarrow(x_2,y_2)\rightarrow(x_2,y_1)\rightarrow(x_1,y_1)$
(which must equal~1).%
\label{defn:local}
\end{defn}

\begin{defn}
Given $p(x|y)$ and $p(y|x)$,
a \textbf{non-local compatibility constraint}
is a probability ratio of a loop in the non-zero graph,
that cannot be expressed as a product of local constraints.%
\label{defn:non-local}
\end{defn}

By \cref{lem:loop-ratio},
strong compatibility requires all edges in the non-zero graph
to have reverse edges,
and weak compatibility only depends on the strongly connected components
without outgoing edges to other components.
We therefore restrict to considering graphs
where all reverse edges exist,
and so we can view the non-zero graph as undirected.
Local constraints are easy to find,
and we can formalise the idea of factoring them out using homology theory.
It simplifies the presentation to use polygonal complexes,
which generalise simplicial complexes:

\begin{defn}
Given $p(x|y)$ and $p(y|x)$,
where all reverse edges exist in the non-zero graph,
their \textbf{non-zero polygonal complex}
is defined with graph nodes as polyhedral vertices,
graph edges as polyhedral edges,
a triangular face when all three $\{(x,y_1),(x,y_2),(x,y_3)\}$
or $\{(x_1,y),(x_2,y),(x_3,y)\}$ exist,
and a square face when all four $\{(x_1,y_1),(x_1,y_2),(x_2,y_1),(x_2,y_2)\}$ exist.%
\label{defn:polygonal}
\end{defn}

\begin{lem}
Given $p(x|y)$ and $p(y|x)$,
where all reverse edges exist in the non-zero graph,
the non-zero polygonal complex is a well-defined polygonal complex.%
\label{lem:polygonal}
\end{lem}

\begin{proof}
We need to show that all non-empty intersections of faces are edges or vertices.
All singleton subsets of a face are vertices by definition.
If both $(x,y_1)$ and $(x,y_2)$ exist,
then by definition of the non-zero graph, there is an edge between them
(similarly for $(x_1,y)$ and $(x_2,y)$);
therefore all non-empty intersections with a triangular face
give a vertex or edge.
Each square face $\{(x_1,y_1),(x_1,y_2),(x_2,y_1),(x_2,y_2)\}$
contains exactly four edges (keeping one $x$ or $y$ value fixed),
again by definition of the non-zero graph.
A non-empty intersection of two square faces
must share at least one $x$ value and one $y$ value;
if exactly one of each, this gives a vertex;
if two of one and one of the other, this gives an edge.
\end{proof}

Each local constraint corresponds to a square face.
Each triangular face formally has a constraint,
but it only includes terms from one conditional
and is trivially satisfied.
The triangular faces are nonetheless necessary
when considering independence of constraints:
they encode the fact that edge ratios cannot be set independently,
but rather must derive from the same conditional distribution.

\begin{thm}
Given $p(x|y)$ and $p(y|x)$,
where all reverse edges exist in the non-zero graph,
the number of independent non-local compatibility constraints is equal to
the rank of the first homology group of the non-zero polygonal complex.%
\label{thm:rank}
\end{thm}

\begin{proof}
Probability ratios form a group under multiplication:
every product of ratios is a ratio,
and every ratio has an inverse (its reciprocal).
Furthermore, every ratio can be expressed in terms of an integer exponent
for each edge in the non-zero graph
(with the sign of the integer corresponding to
the direction in which the edge is followed),
which is exactly the definition of a homological 1-chain.
Compatibility constraints are ratios of loops,
and so are precisely the 1-chains without boundary.
Local constraints are precisely the boundaries of square faces.

Multiplication of ratios is isomorphic to addition of 1-chains.
In other words, expressing a constraint as a product of other constraints
is equivalent to expressing a 1-chain as a sum of other 1-chains.
In particular, the constraints expressible as a product of local constraints
are isomorphic to the boundary subgroup generated by
the faces of the polygonal complex.
Equivalence classes of non-local constraints
are isomorphic to the first homology group
(factoring out the boundary subgroup).
A maximal set of independent non-local constraints
is precisely a minimal generating set of the group,
and so the number of independent constraints is the rank of the group.
\end{proof}

\begin{rem*}
To solve the compatibility problem,
it is sufficient to check a generating set of constraints.
For any equivalence class of non-local constraints,
we can check a constraint with minimum length.
This suggests the following PSPACE algorithm,
which iterates through constraints more explicitly
than the proof in~\cref{sec:zero-upper}.
For every set of three points,
we can run the st-connectivity algorithm
between each pair of points,
to find a minimum-length loop containing those points.
The st-connectivity algorithm can be modified, still in PSPACE,
to also return the probability ratio
(for compatibility with tractable values)
or the number of times a prime appears in the ratio
(for compatibility in the general case).
\end{rem*}

\subsection{Constructing Hard Problem Instances}
\label{sec:high-rank}

The aim of this section is to construct problem instances where
independent non-local constraints
are large in both number and length,
both of which contribute to the complexity of the problem.
However, we first look at simpler examples.
All reverse edges exist
iff the conditionals have the same non-zero values.
We present the non-zero structure as a matrix,
writing a dot ($\cdot$) to denote a non-zero value
(a node in the non-zero graph).

The left-hand example in~\cref{eqn:two-three-gen}
corresponds to~\cref{eqn:eg-incomp}.
Its polygonal complex is a hexagonal loop,
so its first homology group is of rank~1.
All three zeroes are necessary for it to have a non-trivial homology group:
if we make one of them non-zero,
the polygonal complex is equivalent to a disc
(the hexagon is filled with two squares and two triangles).
The middle and right-hand examples in~\cref{eqn:two-three-gen}
have rank 2 and 3, respectively.
These three examples can be generalised to $n$ values
with rank $n-2$.
For $n>3$, there are more than $n-2$ constraints
that do not repeat any points,
but they can all be generated by a (non-unique) set of size~$n-2$.
\begin{align}
\begin{pmatrix}
\cdot&\cdot&0\\
0&\cdot&\cdot\\
\cdot&0&\cdot\\
\end{pmatrix}
&&
\begin{pmatrix}
\cdot&\cdot&\cdot&0\\
0&0&\cdot&\cdot\\
0&\cdot&0&\cdot\\
\cdot&0&0&\cdot\\
\end{pmatrix}
&&
\begin{pmatrix}
\cdot&\cdot&\cdot&\cdot&0\\
0&0&0&\cdot&\cdot\\
0&0&\cdot&0&\cdot\\
0&\cdot&0&0&\cdot\\
\cdot&0&0&0&\cdot
\end{pmatrix}
\label{eqn:two-three-gen}
\end{align}

The examples in~\cref{eqn:one-long-gen} both have rank~1,
homology-equivalent to a circle.
The left-hand example has no local constraints;
its polygonal complex is a closed chain with no faces,
and its homology group is generated by
a loop of path length 10.
This construction can be generalised to $n$ values
with rank~1,
where the generating loop is of length~$2n$.
(It can also be used to encode a logical AND gate,
while the construction in~\cref{eqn:two-three-gen}
can be used to encode a logical OR gate,
both as seen in~\cref{sec:tract}.)

The right-hand example in~\cref{eqn:one-long-gen}
has 5 independent local constraints;
its polygonal complex is equivalent to a Möbius strip
(comprising five squares, each joined to adjacent squares by two triangles,
but with a twist at each step).
Its homology group is generated by an equivalence class of loops,
of path length between 6 and 10
(ignoring loops that include any $x$ or $y$ value more than twice).
All of these loops must have the same probability ratio,
if all local constraints are satisfied.
\begin{align}
\begin{pmatrix}
\cdot&\cdot&0&0&0\\
0&\cdot&\cdot&0&0\\
0&0&\cdot&\cdot&0\\
0&0&0&\cdot&\cdot\\
\cdot&0&0&0&\cdot
\end{pmatrix}
&&
\begin{pmatrix}
\cdot&\cdot&0&0&\cdot\\
\cdot&\cdot&\cdot&0&0\\
0&\cdot&\cdot&\cdot&0\\
0&0&\cdot&\cdot&\cdot\\
\cdot&0&0&\cdot&\cdot
\end{pmatrix}
\label{eqn:one-long-gen}
\end{align}

With more values, more intricate structures are possible.
The left-hand example~\cref{eqn:triple-ring} has rank~4,
with no local constraints.
There are several ways to see that it has rank 4.
Most simply, it is equivalent to three circles joined in a ring:
three copies of the left-hand example in~\cref{eqn:two-three-gen},
joined together (with two triangles at each of the three joins).
Alternatively, it is equivalent to a 4-punctured projective plane
(we can fill in 3 hexagonal holes, creating a Möbius strip,
and fill the remaining twelve-sided hole to give a projective plane),
or a 3-punctured torus
(we can fill 3 ten-sided holes).
These alternative views are possible because
there are many points adjoining more than two boundary edges.

We can also see the left of~\cref{eqn:triple-ring} as
having a long ``spine'' like the examples in~\cref{eqn:one-long-gen},
with 4 additional points that each increase the rank by~1.
With 7 values, we can fit a whole diagonal of such points,
without creating any square faces,
as shown on the right of~\cref{eqn:triple-ring},
which has rank~8.
\begin{align}
\begin{pmatrix}
\cdot&\cdot&0&0&\cdot&0\\
0&\cdot&\cdot&0&0&0\\
\cdot&0&\cdot&\cdot&0&0\\
0&0&0&\cdot&\cdot&0\\
0&0&\cdot&0&\cdot&\cdot\\
\cdot&0&0&0&0&\cdot
\end{pmatrix}
&&
\begin{pmatrix}
\cdot&\cdot&0&0&0&\cdot&0\\
0&\cdot&\cdot&0&0&0&\cdot\\
\cdot&0&\cdot&\cdot&0&0&0\\
0&\cdot&0&\cdot&\cdot&0&0\\
0&0&\cdot&0&\cdot&\cdot&0\\
0&0&0&\cdot&0&\cdot&\cdot\\
\cdot&0&0&0&\cdot&0&\cdot
\end{pmatrix}
\label{eqn:triple-ring}
\end{align}

With 15 values, we can fit four diagonals without any square faces
(two diagonals giving the spine,
plus two additional diagonals),
giving rank~31,
as shown in~\cref{eqn:second-diag}.
We can continue this construction:
to fit $k$ diagonals,
we need $2^k-1$ values,
giving rank~$(k-2)(2^k-1)+1$.
More precisely, labelling the rows from $1$ to $2^k-1$,
the first column is non-zero precisely at each row $2^j-1$.
Asymptotically, for $n$ values,
this has rank $n\log n$,
which is large, but nonetheless much smaller than the $O(n^6)$
possible loops of length~6.
\begin{equation}
\begin{pmatrix}
\cdot&\cdot&0&0&0&0&0&0&0&\cdot&0&0&0&\cdot&0\\
0&\cdot&\cdot&0&0&0&0&0&0&0&\cdot&0&0&0&\cdot\\
\cdot&0&\cdot&\cdot&0&0&0&0&0&0&0&\cdot&0&0&0\\
0&\cdot&0&\cdot&\cdot&0&0&0&0&0&0&0&\cdot&0&0\\
0&0&\cdot&0&\cdot&\cdot&0&0&0&0&0&0&0&\cdot&0\\
0&0&0&\cdot&0&\cdot&\cdot&0&0&0&0&0&0&0&\cdot\\
\cdot&0&0&0&\cdot&0&\cdot&\cdot&0&0&0&0&0&0&0\\
0&\cdot&0&0&0&\cdot&0&\cdot&\cdot&0&0&0&0&0&0\\
0&0&\cdot&0&0&0&\cdot&0&\cdot&\cdot&0&0&0&0&0\\
0&0&0&\cdot&0&0&0&\cdot&0&\cdot&\cdot&0&0&0&0\\
0&0&0&0&\cdot&0&0&0&\cdot&0&\cdot&\cdot&0&0&0\\
0&0&0&0&0&\cdot&0&0&0&\cdot&0&\cdot&\cdot&0&0\\
0&0&0&0&0&0&\cdot&0&0&0&\cdot&0&\cdot&\cdot&0\\
0&0&0&0&0&0&0&\cdot&0&0&0&\cdot&0&\cdot&\cdot\\
\cdot&0&0&0&0&0&0&0&\cdot&0&0&0&\cdot&0&\cdot
\end{pmatrix}
\label{eqn:second-diag}
\end{equation}

In the above construction, the homology group
is generated by constraints of path length~6,
i.e.\ the shortest possible non-local constraints.
To increase the minimum generator length,
we can again start from a long spine,
and add a diagonal
which is not only further from the spine,
but also has gaps.
This construction generalises the right-hand of~\cref{eqn:triple-ring},
which can be seen as 7 overlapping loops of length~$6$.
We have slightly different behaviour depending on whether
the generator length is a multiple of~4.
For length~$4a+2$ (where $a>2$), we can overlap $4a+1$ loops.
This adds a diagonal at a distance of $2a$ from the spine,
with a non-zero value every $a$ places.
This gives $a(4a+1)$ values, with rank $4a+2$
and generator length~$4a+2$.
The case $a=2$ is shown in~\cref{eqn:10-square}.

For length~$4a$ (where $a>2$), we can overlap $4a$ loops,
which adds one diagonal either side of the spine,
both at a distance of $2a-1$,
and both with a non-zero value every $2a-1$,
but out of phase with each other.
This gives $2a(2a-1)$ values, with rank $4a+1$
and generator length~$4a$.
The case $a=2$ is shown in~\cref{eqn:8-square}.
\begin{gather}
\begin{pmatrix}
\cdot&\cdot&0&0&0&0&0&0&0&0&0&0&0&0&\cdot&0&0&0\\
0&\cdot&\cdot&0&0&0&0&0&0&0&0&0&0&0&0&0&0&0\\
0&0&\cdot&\cdot&0&0&0&0&0&0&0&0&0&0&0&0&\cdot&0\\
0&0&0&\cdot&\cdot&0&0&0&0&0&0&0&0&0&0&0&0&0\\
\cdot&0&0&0&\cdot&\cdot&0&0&0&0&0&0&0&0&0&0&0&0\\
0&0&0&0&0&\cdot&\cdot&0&0&0&0&0&0&0&0&0&0&0\\
0&0&\cdot&0&0&0&\cdot&\cdot&0&0&0&0&0&0&0&0&0&0\\
0&0&0&0&0&0&0&\cdot&\cdot&0&0&0&0&0&0&0&0&0\\
0&0&0&0&\cdot&0&0&0&\cdot&\cdot&0&0&0&0&0&0&0&0\\
0&0&0&0&0&0&0&0&0&\cdot&\cdot&0&0&0&0&0&0&0\\
0&0&0&0&0&0&\cdot&0&0&0&\cdot&\cdot&0&0&0&0&0&0\\
0&0&0&0&0&0&0&0&0&0&0&\cdot&\cdot&0&0&0&0&0\\
0&0&0&0&0&0&0&0&\cdot&0&0&0&\cdot&\cdot&0&0&0&0\\
0&0&0&0&0&0&0&0&0&0&0&0&0&\cdot&\cdot&0&0&0\\
0&0&0&0&0&0&0&0&0&0&\cdot&0&0&0&\cdot&\cdot&0&0\\
0&0&0&0&0&0&0&0&0&0&0&0&0&0&0&\cdot&\cdot&0\\
0&0&0&0&0&0&0&0&0&0&0&0&\cdot&0&0&0&\cdot&\cdot\\
\cdot&0&0&0&0&0&0&0&0&0&0&0&0&0&0&0&0&\cdot
\end{pmatrix}
\label{eqn:10-square}\\
\begin{pmatrix}
\cdot&\cdot&0&0&0&0&0&0&0&\cdot&0&0\\
0&\cdot&\cdot&0&0&\cdot&0&0&0&0&0&0\\
0&0&\cdot&\cdot&0&0&0&0&0&0&0&0\\
\cdot&0&0&\cdot&\cdot&0&0&0&0&0&0&0\\
0&0&0&0&\cdot&\cdot&0&0&\cdot&0&0&0\\
0&0&0&0&0&\cdot&\cdot&0&0&0&0&0\\
0&0&0&\cdot&0&0&\cdot&\cdot&0&0&0&0\\
0&0&0&0&0&0&0&\cdot&\cdot&0&0&\cdot\\
0&0&0&0&0&0&0&0&\cdot&\cdot&0&0\\
0&0&0&0&0&0&\cdot&0&0&\cdot&\cdot&0\\
0&0&\cdot&0&0&0&0&0&0&0&\cdot&\cdot\\
\cdot&0&0&0&0&0&0&0&0&0&0&\cdot
\end{pmatrix}
\label{eqn:8-square}
\end{gather}

For $n$ values,
this construction gives rank and generator length
both asymptotically $2\sqrt{n}$.
In the succinct case, we have $n=2^k$ for $k$-dimensional inputs,
so rank and generator length are
both asymptotically $2^{\frac{1}{2}k+1}$.

\section{Discussion}

Trade-offs between computational tractability and other model properties
are typically only understood informally,
but the need to maintain tractability
often drives researchers to rely on
models that come with no guarantees on their performance.
In a rare example of an explicit discussion of a trade-off,
\citet{bruinsma2023conditional}
consider probabilistic neural models for time-series data
(including applications in environmental modelling,
where physics-based simulations have extremely high computational cost).
They articulate a set of four desirable properties,
one of which is tractability of the objective function used in training,
two of which are types of model flexibility,
and the last of which is coherence of model predictions.
For each combination of three properties,
they give an example of a model which
achieves those three properties but fails on the fourth,
and they observe that no known model can achieve all four at once.
This shows that a trade-off exists in practice,
but without a theoretical justification.
The intractability of the succinct compatibility problem,
as proved in this paper,
demonstrates that high-dimensional probabilistic models
are subject to a trade-off between
computational tractability,
model flexibility,
and probabilistic coherence.

The trade-off has different implications in different settings.
If there are strong reasons for believing
the structure of the model is correct
(for example, models based on physical laws),
we might prioritise flexibility and coherence,
and therefore seek to make inferences despite the computational cost
\citex{for a survey, see}{cranmer2020simulation}.

Alternatively, we might prioritise tractability and coherence.
The results in this paper rely on using arithmetic circuits
that are flexible enough to simulate boolean circuits.
Restricting the class of arithmetic circuit,
for example using Probabilistic Circuits
\citep{vergari2021atlas},
can allow Bayesian inference to be tractable.

Finally, we might prioritise tractability and flexibility,
which are important in many applications of machine learning,
including in cognitive science,
where it is an open question to explain
how all aspects of cognition are computationally tractable
\citep{vanrooij2008tractable,colombo2023mind}.
The results of this paper imply that
for flexible tractable models,
we must expect to sometimes observe incoherence between conditionals.

Furthermore, in many settings
we are primarily interested
in making predictions conditioned on observations,
which requires the conditionals rather than the joint.
The \hyperlink{cor}{Corollary} demonstrates that
constructing tractable conditionals via tractable joints
is itself a restriction on flexibility.
In other words, there is a trade-off between tractability of the joint
and flexibility of the conditionals.
If we prioritise tractability and flexibility of the conditionals,
this provides a strong theoretical motivation
for developing models directly in terms of conditionals,
rather than the traditional approach of first defining a joint.

\bibliography{compatibility}
\bibliographystyle{jlm}

\end{document}